\documentclass[sigconf]{aamas}  
\usepackage{balance}  

\usepackage{booktabs}

\usepackage{algorithm}
\usepackage{algpseudocode}

\usepackage{graphicx}
\usepackage{subcaption}

\usepackage{array}
\newcolumntype{P}[1]{>{\centering\arraybackslash}p{#1}}

\setcopyright{ifaamas}  
\acmDOI{}  
\acmISBN{}  
\acmConference[AAMAS'19]{Proc.\@ of the 18th International Conference on Autonomous Agents and Multiagent Systems (AAMAS 2019)}{May 13--17, 2019}{Montreal, Canada}{N.~Agmon, M.~E.~Taylor, E.~Elkind, M.~Veloso (eds.)}  
\acmYear{2019}  
\copyrightyear{2019}  
\acmPrice{}  

\begin{document}

\title{Online Abstraction with MDP Homomorphisms \\ for Deep Learning}  
\titlenote{This research was partially supported by grant no. GA18-18080S of the Grant Agency of the Czech Republic, grant no. EF15\_003/0000421 of the Ministry of Education, Youth and Sports of the Czech Republic and by the National Science Foundation through IIS-1427081, IIS-1724191, and IIS-1724257, NASA through NNX16AC48A and NNX13AQ85G, ONR through N000141410047, Amazon through an ARA to Platt, and Google through a FRA to Platt. We thank Tokine Atsuta, Brayan Impata and Nao Ouyang for useful feedback on the manuscript.}

\author{Ondrej Biza}
\orcid{0000-0003-3390-8050}
\affiliation{
  \institution{Czech Technical University}
  \city{Prague} 
  \country{Czech Republic}
}
\email{bizaondr@fit.cvut.cz}
\author{Robert Platt}
\affiliation{
  \institution{Northeastern University}
  \city{Boston}
  \state{MA}
  \country{USA}
}
\email{rplatt@ccs.neu.edu}

\begin{abstract}
Abstraction of Markov Decision Processes is a useful tool for solving complex problems, as it can ignore unimportant aspects of an environment, simplifying the process of learning an optimal policy. In this paper, we propose a new algorithm for finding abstract MDPs in environments with continuous state spaces. It is based on MDP homomorphisms, a structure-preserving mapping between MDPs. We demonstrate our algorithm's ability to learn abstractions from collected experience and show how to reuse the abstractions to guide exploration in new tasks the agent encounters. Our novel task transfer method outperforms baselines based on a deep Q-network in the majority of our experiments. The source code is at~\url{https://github.com/ondrejba/aamas_19}.
\end{abstract}


\maketitle


\section{Introduction}

The ability to create useful abstractions automatically is a critical tool for an autonomous agent. Without this, the agent is condemned to plan or learn policies at a relatively low level of abstraction, and it becomes hard to solve complex tasks. What we would like is the ability for the agent to learn new skills or abstractions over time that gradually increase its ability to solve challenging tasks. This paper explores this in the context of reinforcement learning.

There are two main approaches to abstraction in reinforcement learning: temporal abstraction and state abstraction. In temporal abstraction, the agent learns multi-step skills, i.e. policies for achieving subtasks. In state abstraction, the agent learns to group similar states together for the purposes of decision making. For example, for handwriting a note, it may be irrelevant whether the agent is holding a pencil or a pen. In the context of the Markov Decision Process (MDP), state abstraction can be understood using an elegant approach known as the MDP homomorphism framework~\cite{ravindran2004}. An MDP homomorphism is a mapping from the original MDP to a more compact MDP that preserves the important transition and reward structure of the original system. Given an MDP homomorphism to a compact MDP, one may solve the original problem by solving the compact MDP and then projecting those solutions back onto the original problem. Figure~\ref{fig:fig_B} illustrates this in the context of a toy-domain puck stacking problem. The bottom left of Figure~\ref{fig:fig_B} shows two pucks on a $4 \times 4$ grid. The agent must pick up one of the pucks (bottom middle of Figure~\ref{fig:fig_B}) and place it on top of the other puck (bottom right of Figure~\ref{fig:fig_B}). The key observation to make here is that although there are many different two-puck configurations (bottom right of Figure~\ref{fig:fig_B}), they are all equivalent in the sense that the next step is for the agent to pick up one of the pucks. In fact, for puck stacking, the entire system can be summarized by the three-state MDP shown at the top of Figure~\ref{fig:fig_B}. This compact MDP is clearly a useful abstraction for this problem.

\begin{figure}[t]
    \includegraphics[width=0.45\textwidth]{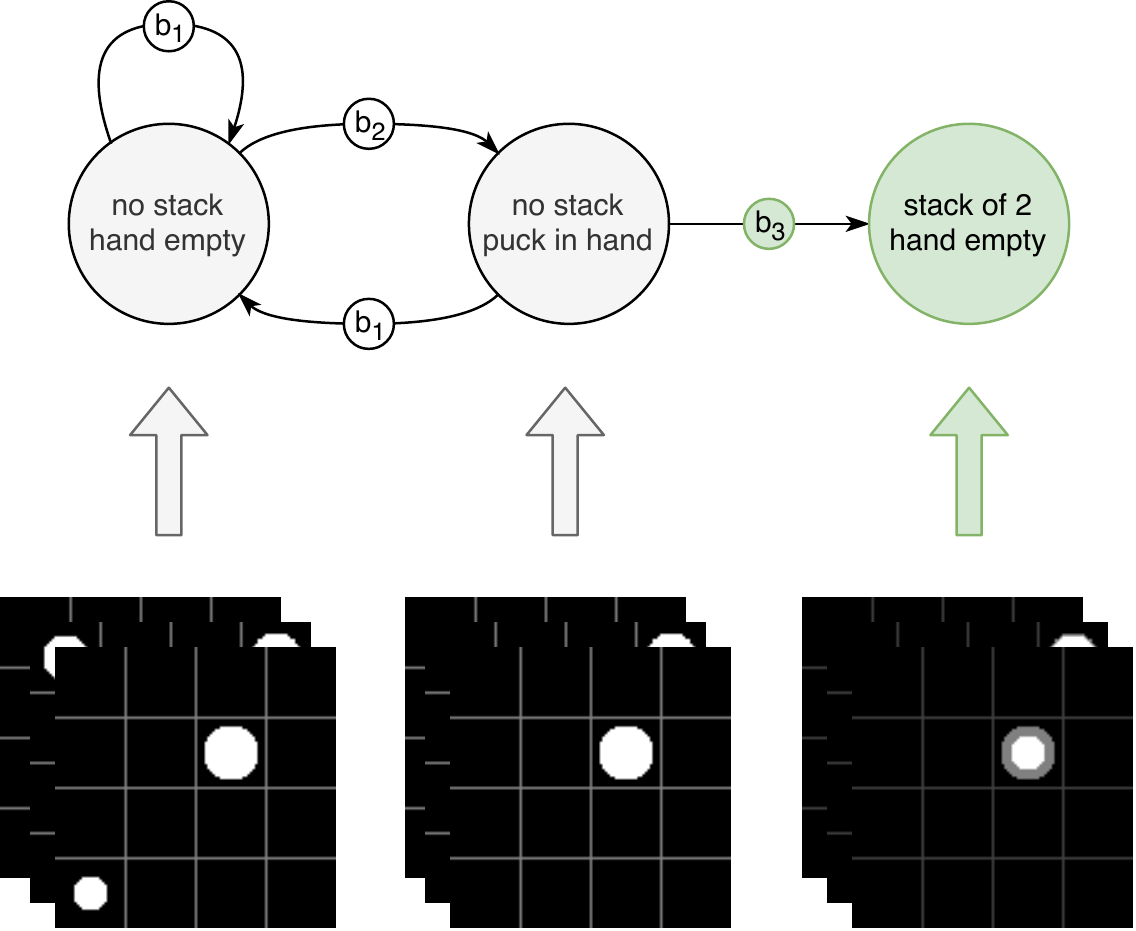}
    \caption{Abstraction for the task of stacking two pucks on top of one another. The diagram shows a minimal quotient MDP (top) that is homomorphic to the underlying MDP (bottom). The minimal MDP has three states, the last of them being the goal state, and four actions. Each action is annotated with the state-action block that induced it. Two actions are annotated with $b_1$ because they both lead to the first state.}
    \label{fig:fig_B}
\end{figure}

Although MDP homomorphisms are a useful mechanism for abstraction, it is not yet clear how to learn the MDP homomorphism mapping from experience in a model-free scenario. This is particularly true for a deep reinforcement learning context where the state space is effectively continuous. The closest piece of related work is probably that of~\cite{wolfe2006} who study the MDP homomorphism learning problem in a narrow context. This paper considers the problem of learning general MDP homomorphisms from experience. We make the following key contributions:

\noindent
\textbf{\#1:} We propose an algorithm for learning MDPs homomorphisms from experience in both discrete and continuous state spaces (Subsection \ref{sub_partitioning_algorithm}). The algorithm groups together state-action pairs with similar behaviors, creating a partition of the state-action space. The partition then induces an abstract MDP homomorphic to the original MDP. We prove the correctness of our method in Section \ref{sec_correctness}.

\noindent
\textbf{\#2:} Our abstraction algorithm requires a learning component. We develop a classifier based on the Dilated Residual Network \cite{yu2017} that enables our algorithm to handle medium-sized environments with continuous state spaces. We include several augmentations, such as oversampling the minority classes and thresholding the confidences of the predictions (Subsection \ref{sub_speeding_up}). We test our algorithm in two environments: a set of continuous state space puck stacking tasks, which leverage the convolutional network (Subsection \ref{continuous_pucks_world}), and a discrete state space blocks world task (Subsection \ref{discrete_blocks_world}), which we solve with a decision tree.

\noindent
\textbf{\#3:} We propose a transfer learning method for guiding exploration in a new task with a previously learned abstract MDP (Subsection \ref{sub_transferring}). Our method is based on the framework of options \cite{sutton1999}: it can augment any existing reinforcement learning agent with a new set of temporally-extended actions. The method outperforms two baselines based on a deep Q-network \cite{mnih2015} in the majority of our experiments.

\section{Background}

\subsection{Reinforcement Learning}

An agent's interaction with an environment can be modeled as a Markov Decision Process (MDP, \cite{bellman57}). An MDP is a tuple $\langle S, A, \Phi, P, R \rangle$, where $S$ is the set of states, $A$ is the set of actions, $\Phi \subset S{\times}A$ is the state-action space (the set of available actions for each state), $P(s,a,s')$ is the transition function and $R(s,a)$ is the reward function. 

We use the framework of options \cite{sutton1999} to transfer knowledge between similar tasks. An option $\langle I, \pi, \beta \rangle$ is a temporally extended action: it can be executed from the set of states $I$ and selects primitive actions with a policy $\pi$ until it terminates. The probability of terminating in each state is expressed by $\beta: S \rightarrow [0,1]$ .

\subsection{Abstraction with MDP homomorphisms}

Abstraction in our paper aims to group similar state-action pairs from the state-action space $\Phi$. The grouping can be described as a partitioning of $\Phi$.

\begin{definition}\label{def_partition_of_MDP}

A \textit{partition of an MDP} $M= \langle S,A,\Phi,P,R \rangle $ is a partition of $\Phi$. Given a partition $B$ of $M$, the \textit{block transition probability of} $M$ is the function $T:\Phi \times B|S \rightarrow [0,1]$ defined by $T(s,a,[s']_{B|S}) = \sum_{s''\in[s']_{B|S}}P(s,a,s'')$. 

\end{definition}

\begin{definition}\label{def_refinement}

A partition~$B'$ is a \textit{refinement} of a partition $B$, $B' \ll B$, if and only if each block of $B'$ is a subset of some block of~$B$.

\end{definition}

The partition of $\Phi$ is projected on the state space $S$ to obtain a grouping of states.

\begin{definition}\label{def_projection}

Let $B$ be a partition of $Z \subseteq X \times Y$, where $X$ and $Y$ are arbitrary sets. For any $x \in X$, let $B(x)$ denote the set of distinct blocks of $B$ containing pairs of which $x$ is a component, that is, $B(x) = \{[(w,y)]_B\ |\ (w,y) \in Z, w = x\}$. The \textit{projection of B onto X} is the partition $B|X$ of $X$ such that for any $x,x' \in X$, $[x]_{B|X} = [x']_{B|X}$ if and only if $B(x) = B(x')$.

\end{definition}

Next, we define two desirable properties of a partition over $\Phi$.

\begin{definition}\label{def_reward_respecting}

A partition B of an MDP $M= \langle S,A,\Phi,P,R \rangle $ is said to be \textit{reward respecting} if $(s_1,a_1)\ {\equiv}_B\ (s_2,a_2)$ implies $R(s_1,a_1)=R(s_2,a_2)$ for all $(s_1,a_1),(s_2,a_2)\in\Phi$.

\end{definition}

\begin{definition}\label{def_ssp}

A partition B of an MDP $M= \langle S,A,\Phi,P,R \rangle $ has the \textit{stochastic substitution property} (SSP) if for all $(s_1,a_1),(s_2,a_2)\in\Phi$, $(s_1,a_1)\ {\equiv}_B\ (s_2,a_2)$ implies $T(s_1,a_1,[s]_{B|S}) = T(s_2,a_2,[s]_{B|S})$ for all $[s]_{B|S}\ {\in}\ B|S$.

\end{definition}

Having a partition with these properties, we can construct the \textit{quotient MDP} (we also call it the \textit{abstract MDP}).

\begin{definition}\label{def_quotient_mdp}

\tolerance 1414 Given a reward respecting SSP partition $B$ of an MDP $M = \langle S,A,\Phi,P,R \rangle$, the \textit{quotient MDP} $M/B$ is the MDP $\langle S',A',\Phi',P',R' \rangle$, where $S' = B|S$; $A' = \bigcup\limits_{[s]_{B|S}{\in}S'} A'_{[s]_{B|S}}$ where $A'_{[s]_{B|S}}= \{a'_1, a'_2, ..., a'_{\eta(s)}\}$ for each $[s]_{B|S} \in S'$; $P'$ is given by $P'([s]_f, a'_i, [s']_f) = T_b([(s,a_i)]_B, [s']_{B|S})$ and $R'$ is given by $R'([s]_{B|S}, a'_i) = R(s,a_i)$. $\eta(s)$ is the number of distinct classes of~$B$ that contain a state-action pair with $s$ as the state component.

\end{definition}

We want the quotient MDP to retain the structure of the original MDP while abstracting away unnecessary information. MDP homomorphism formalizes this intuition.

\begin{definition}\label{def_homomorphism}

\tolerance 1414 An \textit{MDP homomorphism} from $M= \langle S,A,\Phi,P,R \rangle $ to $M'= \langle S',A',\Phi',P',R' \rangle $ is a tuple of surjections $ \langle f,\{g_s:s \in S\} \rangle $ with $h(s,a)=(f(s),g_s(a))$, where $f:S \rightarrow S'$ and $g_s:A \rightarrow A'$ such that $R(s,a)=R'(f(s),g_s(a))$ and $P(s,a,f^{-1}(f(s')))=P'(f(s),g_s(a),f(s'))$. We call $M'$ a \textit{homomorphic image} of $M$ under $h$.

\end{definition}

The following theorem states that the quotient MDP defined above retains the structure of the original MDP.

\begin{theorem}[\cite{ravindran2004}]\label{theorem_quotient_is_homomorphic}

Let $B$ be a reward respecting SSP partition of MDP $M = \langle S,A,\Phi,P,R \rangle$. The quotient MDP $M/B$ is a homomorphic image of $M$.

\end{theorem}

\tolerance 1414 Computing the optimal state-action value function in the quotient MDP usually requires fewer computations, but does it help us act in the underlying MDP? The last theorem states that the optimal state-action value function \textit{lifted} from the minimized MDP is still optimal in the original MDP:

\begin{theorem}[Optimal value equivalence, \cite{ravindran2004}]\label{theorem_optimal_value}

Let $M'= \langle S', A', \Phi', P', R' \rangle$ be the homomorphic image of the MDP $M= \langle S, A, \Phi, P, R \rangle$ under the MDP homomorphism $h(s,a)=(f(s),g_s(a))$. For any $(s,a) \in \Phi$, $Q^*(s,a)=Q^*(f(s),g_s(a))$.

\end{theorem}

\section{Related Work}

Balaraman Ravindran proposed Markov Decision Process (MDP) homomorphism together with a sketch of an algorithm for finding homomorphisms (i.e. finding the minimal MDP homomorphic to the underlying MDP) given the full specification of the MDP in his Ph.D. thesis \cite{ravindran2004}. The first and only algorithm (to the best of our knowledge) for finding homomorphisms from experience (online) \cite{wolfe2006} operates over Controlled Markov Processes (CMP), an MDP extended with an output function that provides more supervision than the reward function alone. Homomorphisms over CMPs were also used in \cite{wolfe2006b} to find objects that react the same to a defined set of actions.

An approximate MDP homomorphism \cite{ravindran2004approximate} allows aggregating state-action pairs with similar, but not the same dynamics. It is essential when learning homomorphisms from experience in non-deterministic environments because the estimated transition probabilities for individual state-action pairs will rarely be the same, which is required by the MDP homomorphism. Taylor et al. \cite{taylor2008} built upon this framework by introducing a similarity metric for state-action pairs as well as an algorithm for finding approximate homomorphisms.


Sorg et al. \cite{sorg2009} developed a method based on homomorphisms for transferring a predefined optimal policy to a similar task. However, their approach maps only states and not actions, requiring actions to behave the same across all MDPs. Soni et al. and Rajendran et al. \cite{soni2006, rajendran2009} also studied skill transfer in the framework of MDP homomorphisms. Their works focus on the problem of transferring policies between discrete or factored MDPs with pre-defined mappings, whereas our primary contribution is the abstraction of MDPs with continuous state spaces.

\begin{algorithm}[t]

    \caption{Abstraction}\label{abstraction}
    
    \begin{algorithmic}[1]
        
        \Procedure{Abstraction}{}
        
            \State $E \gets $ collect initial experience with an arbitrary policy $\pi$
            \State $g \gets $ a classifier for state-action pairs
            \State $B \gets OnlinePartitionIteration(E,g)$
            \State $M' \gets $ a quotient MDP constructed from $B$ according to Definition \ref{def_quotient_mdp}

        \EndProcedure
    \end{algorithmic}

\end{algorithm}

\section{Methods}

We solve the problem of abstracting an MDP with a discrete or continuous state-space and a discrete action space. The MDP can have an arbitrary reward function, but we restrict the transition function to be deterministic. This restriction simplifies our algorithm and makes it more sample-efficient (because we do not have to estimate the transition probabilities for each state-action pair).

This section starts with an overview of our abstraction process (Subsection \ref{abstraction}), followed by a description of our algorithm for finding MDP homomorphisms (Subsection \ref{sub_partitioning_algorithm}). We describe several augmentations to the base algorithm that increase its robustness in Subsection \ref{sub_speeding_up}. Finally, Subsection \ref{sub_transferring} contains the description of our transfer learning method that leverages the learned MDP homomorphism to speed up the learning of new tasks. 

\subsection{Abstraction}\label{sub_abstraction}

Algorithm \ref{abstraction} gives an overview of our abstraction process. Since we find MDP homomorphisms from experience, we first need to collect transitions that cover all regions of the state-action space. For simple environments, a random exploration policy provides such experience. But, a random walk is clearly not sufficient for more realistic environments because it rarely reaches the goal of the task. Therefore, we use the vanilla version of a deep Q-network~\cite{mnih2015} to collect the initial experience in bigger environments.

Subsequently, we partition the state-action space of the original MDP based on the collected experience with our Online Partition Iteration algorithm (Algorithm \ref{partition_iteration}). The algorithm is described in detail in Subsection \ref{sub_partitioning_algorithm}. The state-action partition $B$--the output of Algorithm \ref{partition_iteration}--induces a quotient, or abstract, MDP according to Definition \ref{def_quotient_mdp}.

The quotient MDP enables both planning optimal actions for the current task (Subsection \ref{sub_partitioning_algorithm}) and learning new tasks faster (Subsection \ref{sub_transferring}).

\begin{algorithm}[t]

    \caption{Online Partition Iteration}\label{partition_iteration}
    
    \begin{flushleft}
        \hspace*{\algorithmicindent} \textbf{Input:} Experience $E$, classifier $g$. \\
        \hspace*{\algorithmicindent} \textbf{Output:} Reward respecting SSP partition $B$. \\
    \end{flushleft}
    
    \begin{algorithmic}[1]
    
        \Procedure{OnlinePartitionIteration}{}
        
            \State $B \gets \{E\}, B' \gets \{\}$
            \State $B \gets SplitRewards(B)$

            \While {$B \ne B'$}
            
                \State $B' \gets B$
                \State $g \gets TrainClassifier(B, g)$
                \State $B|S \gets Project(B, g)$
                \For {block $c$ in $B|S$}
                    \While {$B$ contains block $b$ for which $B \ne Split(b, c, B)$}
                        \State $B \gets Split(b, c, B)$
                    \EndWhile
                \EndFor

            \EndWhile
        
        \EndProcedure
        
    \end{algorithmic}

\end{algorithm}

\subsection{Partitioning algorithm}\label{sub_partitioning_algorithm}

Our online partitioning algorithm (Algorithm \ref{partition_iteration}) is based on the Partition Iteration algorithm from \cite{givan2003}. It was originally developed for stochastic bisimulation based partitioning, and we adapted it to MDP homomorphisms (following Ravindran's sketch \cite{ravindran2004}). Algorithm \ref{sub_partitioning_algorithm} starts with a reward respecting partition obtained by separating transitions that receive distinct rewards (\textit{SplitRewards}). The reward respecting partition is subsequently refined with the \textit{Split} (Algorithm \ref{split}) operation until a stopping condition is met. \textit{Split(b, c, B)} splits a state-action block $b$ from state-action partition $B$ with respect to a state block $c$ obtained by projecting the partition $B$ onto the state space.

The projection of the state-action partition onto the state space (Algorithm \ref{state_projection}) is the most complex component of our method. We train a classifier $g$, which can be an arbitrary model, to classify state-action pairs into their corresponding state-action blocks. The training set consists of all transitions the agent experienced, with each transition belonging to a particular state-action block. During State Projection, $g$ evaluates a state under a sampled set of actions, predicting a state-action block for each action. For discrete action spaces, the set should include all available actions. The set of predicted state-action blocks determines which state block the state belongs to.

Figure \ref{fig:fig_A} illustrates the projection process: a single state $s$ is evaluated under four actions: $a_1$, $a_2$, $a_3$ and $a_4$.  The first three actions are classified into the state-action block $b_1$, whereas the last action is assigned to block $b_3$. Therefore, $s$ belongs to the state block identified by the set of the predicted state-action blocks $\{b_1, b_3\}$.

The output of Online Partition Iteration is a partition $B$ of the state-action space $\Phi$. According to Definition \ref{def_quotient_mdp}, the partition induces a quotient MDP. Since the quotient MDP is fully defined, we can compute its optimal Q-values with a dynamic programming method such as Value Iteration \cite{sutton1998book}.

To be able to act according to the quotient MDP, we need to connect it to the original MDP in which we select actions. Given a current state $s$ and a set of actions admissible in $s$, $A_s$, we predict the state-action block of each pair $(s,a_i)$, $a_i \in A_s$ using the classifier $g$. Note that Online Partition Iteration trains $g$ in the process of refining the partition. This process of predicting state-action block corresponds to a single step of State Projection: we determine which state block $s$ belongs to. Since each state in the quotient MDP corresponds to a single state block (by Definition \ref{def_quotient_mdp}), we can map $s$ to some state $s'$ in the quotient MDP.

Given the current state $s'$ in the quotient MDP, we select the action with the highest Q-value and map it back to the underlying MDP. An action in the quotient MDP can correspond to more than one action in the underlying MDP. For instance, an action that places a puck on the ground can be executed in many locations, while still having the same Q-value in the context of puck stacking. We break the ties between actions by sampling a single action in proportion to the confidence predicted by $g$: $g$ predict a state-action block with some probability given a state-action pair.

\begin{figure}[b]
    \includegraphics[width=0.45\textwidth]{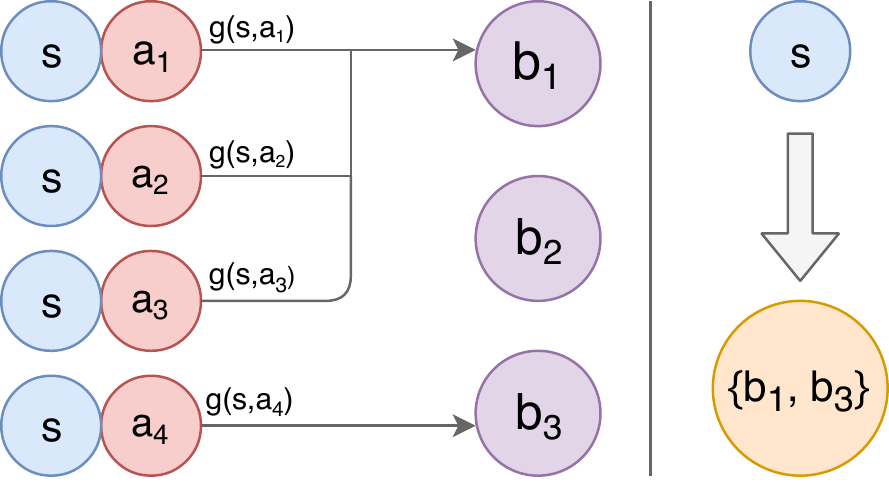}
    \caption{Projection (Algorithm \ref{state_projection}) of a single state $s$. $s$ is evaluated under actions $a_1$, $a_2$, $a_3$ and $a_4$. For each pair $(s,a_i)$, the classifier $g$ predicts its state-action block $b_j$. $s$ belongs to a state block identified by the set of state-action blocks $\{b_1,b_3\}$.}
    \label{fig:fig_A}
\end{figure}

\subsection{Increasing robustness}\label{sub_speeding_up}

\begin{algorithm}[t]

    \caption{State Projection}\label{state_projection}
    
    \begin{flushleft}
        \hspace*{\algorithmicindent} \textbf{Input:} State-action partition $B$, classifier $g$.  \\
        \hspace*{\algorithmicindent} \textbf{Output:} State partition $B|S$. \\
    \end{flushleft}

    \begin{algorithmic}[1]
        \Procedure{Project}{}
            \State $B|S \gets \{\}$
            \For  {block $b$ in $B$}
                \For {transition $t$ in $b$}
                    \State $A_s \gets SampleActions(t.next\_state)$
                    \State $B_s \gets \{\}$
                    \For {action $a$ in $A_s$}
                        \State $p \gets g.predict(t.next\_state, a)$
                        \State $B_s \gets B_s \cup \{p\}$
                    \EndFor
                    \State $Add\ t\ to\ B|S\ using\ B_s\ as\ the\ key$
                \EndFor
            \EndFor
        \EndProcedure
    \end{algorithmic}

\end{algorithm}

Online Partition Iteration is sensitive to erroneous predictions by the classifier $g$. Since the collected transitions tend to be highly unbalanced and the mapping of state-action pairs into state-action blocks can be hard to determine, we include several augmentations that increase the robustness of our method. Some of them are specific to a neural network classifier.

\begin{itemize}

\item \textbf{class balancing}: The sets of state-action pairs belonging to different state-action blocks can be extremely unbalanced. Namely, the number of transitions that are assigned a positive reward is usually low. We follow the best practices from \cite{buda2018} and over-sample all minority classes so that the number of samples for each class is equal to the size of the majority class. We found decision trees do not require oversampling; hence, we use this method only with a neural network.

\item \textbf{confidence calibration}: The latest improvements to neural networks, such as batch normalization \cite{ioffe15} and skip connections \cite{he16} (both used by our neural network in Subsection \ref{continuous_pucks_world}), can cause miscalibration of the output class probabilities \cite{guo17}. We calibrate the temperate of the softmax function applied to the output neurons using a multiclass version of Platt scaling \cite{platt99} derived in \cite{guo17}. The method requires a held-out validation set, which consists of 20\% of all data in our case.

\item \textbf{state-action block size threshold and confidence threshold}:  During State Projection, the classifier $g$ sometimes makes mistakes in classifying a state-action pair to a state-action block. Hence, the State Projection algorithm can assign a state to a wrong state block. This problems usually manifests itself with the algorithm "hallucinating" state blocks that do not exist in reality (note that there are $2^{min\{|B|, |A|\}} - 1$ possible state blocks, given a state-action partition $B$). To prevent the \textit{Split} function from over-segmenting the state-action partition due to these phantom state blocks, we only split a state-action block if the new blocks contain a number of samples higher than a threshold $T_a$. Furthermore, we exclude all predictions with confidence lower than some threshold $T_c$. Confidence calibration makes it easier to select the optimal value of $T_c$.

\end{itemize}

\begin{algorithm}[t]

    \caption{Split}\label{split}
    
    \begin{flushleft}
        \hspace*{\algorithmicindent} \textbf{Input} State-action block $b$, state block $c$, partition $B$. \\
        \hspace*{\algorithmicindent} \textbf{Output} State-action partition $B'$. \\
    \end{flushleft}
    
    \begin{algorithmic}[1]
        
        \Procedure{Split}{}
        
            \State $b_1 \gets \{\}, b_2 \gets \{\}$

            \For {transition $t$ in $b$}
                \If{$transition.next\_state \in c$}
                    \State $b_1 \gets b_1 \cup \{t\}$
                \Else
                    \State $b_2 \gets b_2 \cup \{t\}$
                \EndIf
            \EndFor
            
            \State $B' \gets B$
            
            \If{$|b_1| > 0\ \&\&\ |b_2| > 0$}
                \State $B' \gets (B' \setminus \{b\}) \cup \{b_1, b_2\}$
            \EndIf
            
        \EndProcedure
    \end{algorithmic}

\end{algorithm}

\subsection{Transferring abstract MDPs}\label{sub_transferring}

Solving a new task from scratch requires the agent to take a random walk before it stumbles upon a reward. The abstract MDP learned in the previous task can guide exploration by taking the agent into a starting state close to the goal of the task. However, how do we know which state block in the abstract MDP is a good start for solving a new task?

If we do not have any prior information about the structure of the next task, the agent needs to explore the starting states. To formalize this, we create $|B|S|$ options, each taking the agent to a particular state in the quotient MDP from the first task. Each option is a tuple $\langle I, \pi, \beta \rangle$ with 

\begin{itemize}
    \item $I$ being the set of all starting states of the MDP for the new task,
    \item $\pi$ uses the quotient MDP from the previous task to select actions that lead to a particular state in the quotient MDP (see Subsection \ref{sub_partitioning_algorithm} for more details) and
    \item $\beta$ terminates the option when the target state is reached.
\end{itemize}

The agent learns the $Q$-values of the options with a Monte Carlo update \cite{sutton1998book} with a fixed $\alpha$ (the learning rate)--the agent prefers options that make it reach the goal the fastest upon being executed. If the tasks are similar enough, the agent will find an option that brings it closer to the goal of the next task. If not, the agent can choose not to execute any option.

We use a deep Q-network to collect the initial experience in all transfer learning experiments. While our algorithm suffers from the same scalability issues as a deep Q-network when learning the \textit{initial task}, our transfer learning method makes the learning of new tasks easier by guiding the agent's exploration.

\section{Proof of Correctness}\label{sec_correctness}

This section contains the proof of the correctness of our algorithm. We first prove two lemmas that support the main theorem. The first lemma and corollary ensure that Algorithm \ref{partition_iteration} finds a reward respecting SSP partition.

\begin{lemma}\label{lemma_1}

Given a reward respecting partition  $B$ of an MDP $M= \langle S,A,\Phi,P,R \rangle $ and $(s_1,a_1),(s_2,a_2) \in \Phi$ such that $T(s_1,a_1,[s']_{B|S}) \ne T(s_2,a_2,[s']_{B|S})$ for some $s' \in S$, $(s_1,a_1)$ and $(s_2,a_2)$ are not in the same block of any reward respecting SSP partition refining $B$.

\end{lemma}

\begin{proof}

Following the proof of Lemma 8.1 from \cite{givan2003}: proof by contradiction.

\sloppy Let $B'$ be a reward respecting SSP partition that is a refinement of $B$. Let $s' \in S$, $(s_1,a_1),(s_2,a_2) \in b \in B$ such that $T(s_1,a_1,[s']_{B|S}) \ne T(s_2,a_2,[s']_{B|S})$. Define $B'$ such that $(s_1,a_1),(s_2,a_2)$ are in the same block and $[s']_{B|S} = \bigcup_{i=1}^k [{s'}_i]_{B'|S}$. Because $B'$ is a reward respecting SSP partition, for each state block $[s'']_{B|S} \in B'|S$, $T(s_1,a_1,[s'']_{B|S}) = T(s_2,a_2,[s'']_{B|S})$. Then, $T(s_1,a_1,[s']_{B|S}) = \sum_{1 \leq i \leq k} T(s_1,a_1,[{s'}_i]_{B'|S}) = \sum_{1 \leq i \leq k} T(s_2,a_2,[{s'}_i]_{B'|S}) = T(s_2,a_2,[s']_{B|S})$. This contradicts $T(s_1,a_1,[s']_{B|S}) \ne T(s_2,a_2,[s']_{B|S})$. 

\end{proof}

\begin{corollary}\label{corollary_1}

Let $B$ be a reward respecting partition of an MDP $M= \langle S,A,\Phi,P,R \rangle $, $b$ a block in $B$ and $c$ a union of blocks from $B|S$. Every reward respecting SSP partition over $\Phi$ that refines $B$ is a refinement of the partition $Split(b,c,B)$.

\end{corollary}

\begin{proof}

Following the proof of Corollary 8.2 from \cite{givan2003}.

Let $c = \bigcup_{i=1}^n [s_i]_{B|S}$, $[s_i]_{B|S} \in {B|S}$. Let $B'$ be a reward respecting SSP partition that refines $B$. $Split(b,c,B)$ will only split state-action pairs $(s_1,a_1),(s_2,a_2)$ if $T(s_1,a_1,c) \ne T(s_2,a_2,c)$. But if $T(s_1,a_1,c) \ne T(s_2,a_2,c)$, then there must be some $k$ such that $T(s_1,a_1,c_k) \ne T(s_2,a_2,c_k)$ because for any $(s,a) \in \Phi$, $T(s,a,c) = \sum_{1 \leq m \leq n} T(s,a,c_m)$. Therefore, we can conclude by Lemma \ref{lemma_1} that $[(s_1,a_1)]_{B'} \ne [(s_2,a_2)]_{B'}$.

\end{proof}

The versions of Partition Iteration from \cite{givan2003} and \cite{ravindran2004} partition a fully-defined MDP. We designed our algorithm for the more realistic case, where only a stream of experience is available. This change makes the algorithm different only during State Projection (Algorithm \ref{state_projection}). In the next lemma, we prove that the output of State Projection converges to a state partition as the number of experienced transitions goes to infinity.

\begin{table*}[!t]
    \scalebox{0.9}{
        \begin{tabular}{P{0.713\columnwidth}|P{0.433\columnwidth}|P{0.433\columnwidth}|P{0.433\columnwidth}}
        \hline 
        Task                               & Options      & Baseline     & Baseline, share weights \\
        \hline 
        2 puck stack to 3 puck stack        & $\mathbf{2558 \pm 910}$ & $5335 \pm 1540$ & $10174 \pm 5855$     \\
        3 puck stack to 2 and 2 puck stack  & $\mathbf{2382 \pm 432}$ & - & $3512 \pm 518$                    \\
        2 puck stack to stairs from 3 pucks & $\mathbf{2444 \pm 487}$ & $4061 \pm 1382$ & $4958 \pm 3514$      \\
        3 puck stack to stairs from 3 pucks & $\mathbf{1952 \pm 606}$ & $4061 \pm 1382$ & $5303 \pm 3609$      \\
        2 puck stack to 3 puck component    & $2781 \pm 605$ & $3394 \pm 999$ & $6641 \pm 5582$      \\
        \hline
        stairs from 3 pucks to 3 puck stack        & $3947 \pm 873$   & $5335 \pm 1540$   & $6563 \pm 4299$   \\
        stairs from 3 pucks to 2 and 2 puck stacks & $5552 \pm 3778$  & -                 & $5008 \pm 1998$   \\
        stairs from 3 pucks to 3 puck component    & $3996 \pm 2693$  & $3394 \pm 999$   & $4856 \pm 3600$   \\
        \hline
        3 puck component to 3 puck stack        & $\mathbf{3729 \pm 742}$ & $5335 \pm 1540$ & $8540 \pm 4908$   \\
        3 puck component to stairs from 3 pucks & $3310 \pm 627$ & $4061 \pm 1382$ & $2918 \pm 328$    \\
        \hline 
        \end{tabular}
    }
\smallskip
\caption{Transfer experiments in the pucks world domain. We measure the number of time steps before the agent reached the goal in at least 80\% of episodes over a window of 50 episodes and report the mean and standard deviation over 10 trials. Unreported scores mean that the agent never reached this target. The column labeled \textit{Options} represents our transfer learning method (Subsection \ref{sub_transferring}), \textit{Baseline} is deep Q-network described in Subsection \ref{continuous_pucks_world} that does not retain any information from the initial task and \textit{Baseline, share weights} copies the trained weights of the network from initial task to the transfer task. The bolded scores correspond to a statistically significant result for a Welch's t-test with P < 0.1.}
\label{table:1}
\end{table*}

\begin{lemma}\label{lemma_2}

Let $M= \langle S,A,\Phi,P,R \rangle $ be an MDP with a finite $A$, a finite or infinite $S$, a state-action space $\Phi$ that is a separable metric space and a deterministic $P$ defined such that each state-action pair is visited with a probability greater than zero. Let $SampleAction(s) = A_s, \forall s \in S$ (Algorithm \ref{state_projection}, line 5). Let $t_1$, $t_2$, ... be i.i.d. random variables that represent observed transitions, $g$ a 1 nearest neighbor classifier that classifies state-action pairs into state-action blocks and let $(s,a)_n$ the nearest neighbor to $(s,a)$ from a set of $n$ transitions $X_n = \{t_1, t_2, ..., t_n\}$. Let $B_n$ be a state-action partition over $X_n$ and $S_n = \bigcup_{t \in X_n} t.next\_state$. Let $(B_n|S_n)'$ be a state partition obtained by the State Projection algorithm with $g$ taking neighbors from $X_n$. $(B_n|S_n)' \rightarrow B_n|S_n$ as $n \rightarrow \infty$ with probability one.

\end{lemma}

\begin{proof}

$B_n|S_n$ is obtained by projecting $B_n$ onto $S_n$. In this process, $S_n$ is divided into blocks based on $B(s) = \{[(s',a)]_{B_n} | (s',a) \in \Phi, s=s'\}$, the set of distinct blocks containing pairs of which $s$ is a component, $s \in S_n$. Given $SampleAction(s) = A_s, \forall s \in S_n$, line 8 in Algorithm \ref{state_projection} predicts a $b \in B$ for each $(s',a) \in \Phi$, such that $s = s'$. By the Convergence of the Nearest Neighbor Lemma \cite{cover67}, $(s,a)_n$ converges to $(s,a)$ with probability one. The rest of the Algorithm \ref{state_projection} exactly follows the projection procedure (Definition 3), therefore, $(B_n|S_n)' \rightarrow B_n|S_n$ with probability one.

\end{proof}

Finally, we prove the correctness of our algorithm given an infinite stream of i.i.d. experience. While the i.i.d. assumption does not usually hold in reinforcement learning (RL), the deep RL literature often leverages the experience buffer \cite{mnih2015} to ensure the training data is diverse enough. Our algorithm also contains a large experience buffer to collect the data needed to run Online Partition Iteration.

\begin{theorem}[Correctness]\label{theorem_correctness}

Let $M= \langle S,A,\Phi,P,R \rangle $ be an MDP with a finite $A$, a finite or infinite $S$, a state-action space $\Phi$ that is a separable metric space and a deterministic $P$ defined such that each state-action pair is visited with a probability greater than zero. Let $SampleAction(s) = A_s, \forall s \in S$ (Algorithm \ref{state_projection}, line 5). Let $t_1$, $t_2$, ... be i.i.d. random variables that represent observed transitions, $g$ a 1 nearest neighbor classifier that classifies state-action pairs into state-action blocks. As the number of observed $t_i$ goes to infinity, Algorithm \ref{partition_iteration} computes a reward respecting SSP partition over the observed state-action pairs with probability one.

\end{theorem}

\begin{proof}

Loosely following the proof of Theorem 8 from \cite{givan2003}.

Let $B$ be a partition over the observed state-action pairs, $S$ the set of observed states and (B|S)' the result of StateProjection(B,g) (Algorithm \ref{state_projection}).

Algorithm \ref{partition_iteration} first splits the initial partition such that a block is created for each set of transitions with a distinct reward (line 2). Therefore, Algorithm \ref{partition_iteration} refines a reward respecting partition from line 2 onward.

Algorithm 1 terminates with $B$ when $B = Split(b,[s]_{(B|S)'},B)$ for all $b \in B$, $[s]_{(B|S)'} \in (B|S)'$. $Split(b,[s]_{(B|S)'},B)$ will split any block $b$ containing $(s_1,a_1),(s_2,a_2)$ for which $T(s_1,a_1,[s]_{(B|S)'}) \ne T(s_2,a_2,[s]_{(B|S)'})$. According to Lemma 2, $(B|S)' \rightarrow B|S$ as $N \rightarrow \infty$ with probability one. Consequently, any partition returned by Algorithm \ref{partition_iteration} must be a reward respecting SSP partition.

Since Algorithm \ref{partition_iteration} first creates a reward respecting partition, and each step only refines the partition by applying $Split$, we can conclude by Corollary 1 that each partition encountered, including the resulting partition, must contain a reward respecting SSP partition.

\end{proof}

\section{Experiments}

We investigate the following questions with our experiments:

\begin{enumerate}
    \item Can Online Partition Iteration find homomorphisms in environments with continuous state spaces and high-dimensional action spaces (characteristic for robotic manipulation tasks)?
    \item Do options induced by quotient MDPs speed-up the learning of new tasks?
    \item How does Online Partition Iteration compare to the only previous approach to finding homomorphisms\cite{wolfe2006}?
\end{enumerate}

Section \ref{continuous_pucks_world} describes our experiments and results concerning question 1 and 2, and section \ref{discrete_blocks_world} presents a comparison with the prior work. We discuss the results in Subsection \ref{exp_discussion}.

\subsection{Continuous pucks world}\label{continuous_pucks_world}

We designed the pucks world domain (Figure \ref{fig:fig_4}) to approximate real-world robotic manipulation tasks. The state is represented by a 112x112 depth image and each pixel in the image is an admissible action. Hence, 12544 actions can be executed in each state. Environments with such a high branching factor favor homomorphisms, as they can automatically group actions into a handful of classes (e.g. "pick puck" and "do nothing") for each state. If an action corresponding to a pixel inside of a puck is selected, the puck is transported in the agent's hand. In the same way, the agent can stack pucks on top of each other or place them on the ground. Corner cases such as placing a puck outside of the environment or making a stack of pucks that would collapse are not allowed. The agent gets a reward of 0 for visiting each non-goal state and a reward of 10 for reaching the goal states. The environment terminates when the goal is reached or after 20 time steps. We implemented four distinct types of tasks: stacking pucks in a single location, making two stacks of pucks, arranging pucks into a connected component and building stairs from pucks. The goal states of the tasks are depicted in Figure \ref{fig:fig_4}. We can instantiate each task type with a different number of pucks, making the space of possible tasks and their combinations even bigger.

\begin{figure}[t]
    \begin{subfigure}[t]{0.23\columnwidth}
        \includegraphics[width=1.0\textwidth]{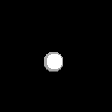}
        \caption{}
    \end{subfigure}
    \begin{subfigure}[t]{0.23\columnwidth}
        \includegraphics[width=1.0\textwidth]{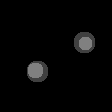}
        \caption{}
    \end{subfigure}
    \begin{subfigure}[t]{0.23\columnwidth}
        \includegraphics[width=1.0\textwidth]{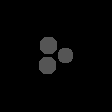}
        \caption{}
    \end{subfigure}
    \begin{subfigure}[t]{0.23\columnwidth}
        \includegraphics[width=1.0\textwidth]{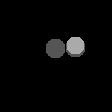}
        \caption{}
    \end{subfigure}
    \caption{The goal states of the four types of tasks in our continuous pucks world domain. a) the task of stacking three pucks on top of one another, b) making two stacks of two pucks, c) arranging three pucks into a connected component, d) building stairs from three pucks.}
    \label{fig:fig_4}
\end{figure}

To gather the initial experience for partitioning, we use a shallow fully-convolutional version of the vanilla deep Q-network. Our implementation is based on the OpenAI baselines \cite{openai_baselines} with the standard techniques: separate target network with a weight update every 100 time steps and a replay buffer that holds the last 10 000 transitions. The network consists of five convolutional layers with the following settings (number of filters, filter sizes, strides): (32, 8, 4), (64, 8, 2), (64, 3, 1), (32, 1, 1), (2, 1, 1). The ReLU activation function is applied to the output of each layer except for the last one. The last layer predicts two maps of Q-values with the resolution 14x14 (for 112x112 inputs)--the two maps correspond to the two possible hand states: "hand full" and "hand empty". The appropriate map is selected based on the state of the hand, and bilinear upsampling is applied to get a 112x112 map of Q-values, one for each action. We trained the network with a Momentum optimizer with the learning rate set to 0.0001 and momentum to 0.9, batch size was set to 32. The agent interacted with the environment for 15000 episodes with an $\epsilon$-greedy exploration policy; $\epsilon$ was linearly annealed from 1.0 to 0.1 for 40000 time steps.

Online Partition Iteration requires a second neural network--the classifier $g$. Our initial experiments showed that the predictions of the architecture described above lack in resolution. Therefore, we chose a deeper architecture: the DRN-C-26 version of Dilated Residual Networks \cite{yu2017}. We observed that the depth of the network is more important than the width (the number of filters in each layer) for our classification task. Capping the number of filters at 32 (the original architecture goes up to 512 filters in the last layers) produces results indistinguishable from the original network. DRN-C-26 decreases the resolution of the feature maps in three places using strided convolutions, we downsample only twice to keep the resolution high. We train the network for 1500 steps during every iteration of Online Partition Iteration. The learning rate for the Momentum optimizer started at 0.1 and was divided by 10 at steps 500 and 1000, momentum was set to 0.9. The batch size was set to 64 and the weight decay to 0.0001.

Figure \ref{fig:fig_5} reports the results of a grid search over state-action block size thresholds and classification confidence thresholds described in Subsection \ref{sub_speeding_up}. Online Partition Iteration can create a near-optimal partition for the three pucks stacking task. On the other hand, our algorithm is less effective in the component arrangement and stairs building tasks. These two tasks are more challenging in terms of abstraction because the individual state-action blocks are not as clearly distinguishable as in puck stacking.

\begin{figure}[t]
    \begin{subfigure}[t]{0.45\columnwidth}
        \includegraphics[width=1.0\textwidth]{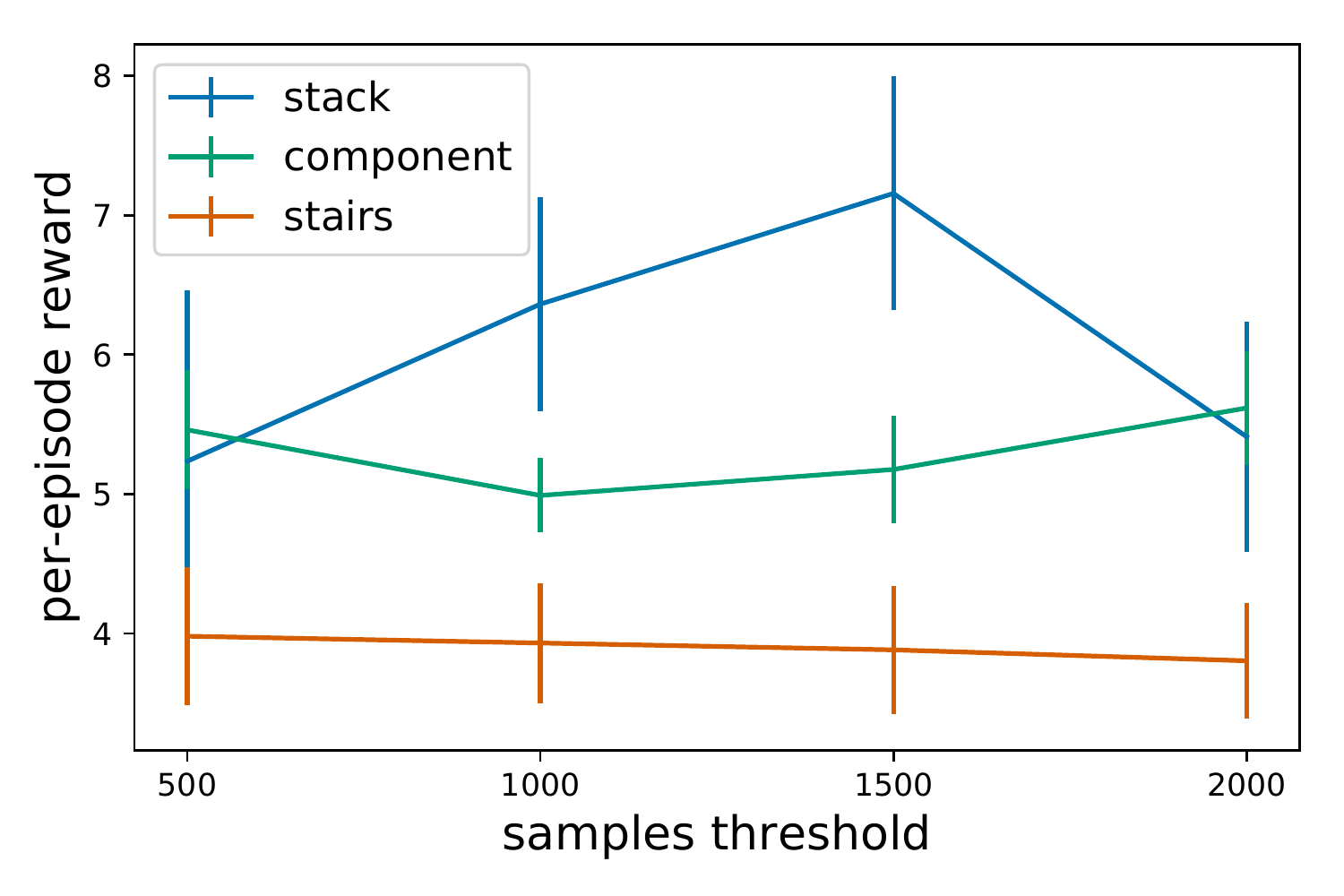}
        \caption{}
    \end{subfigure}
    \hspace{1.5em}
    \begin{subfigure}[t]{0.45\columnwidth}
        \includegraphics[width=1.0\textwidth]{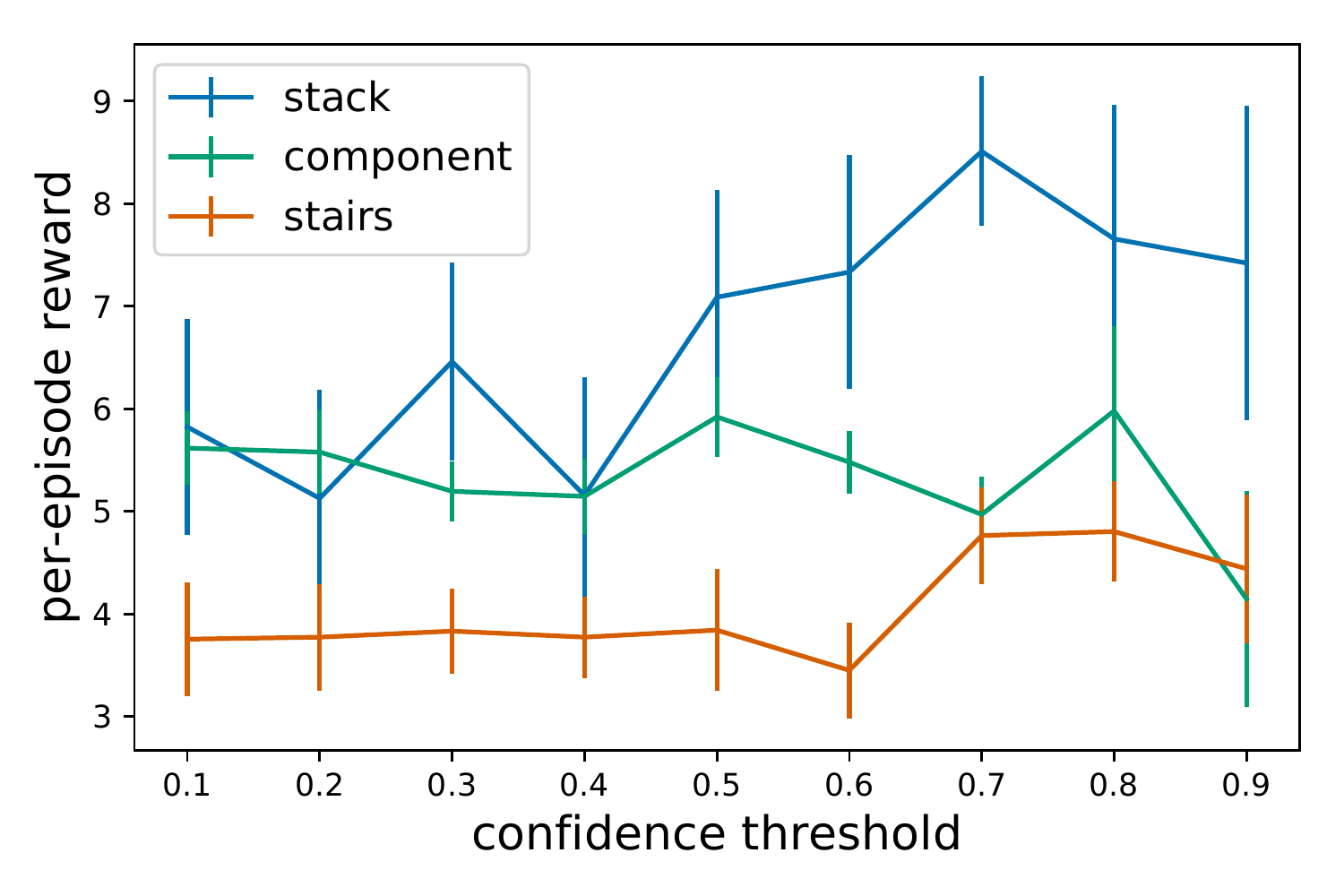}
        \caption{}
    \end{subfigure}
    \caption{Grid searches over state-action block size thresholds and prediction confidence thresholds (described in Subsection \ref{sub_speeding_up}). The y-axis represents the average per-episode reward (the maximum is 10) obtained by planning in the quotient MDP induced by the resulting partition. Stack, component and stairs represent the three puck world tasks shown in Figure \ref{fig:fig_4}. We report the means and standard deviations over 20 runs with different random seeds.}
    \label{fig:fig_5}
\end{figure}

Next, we investigate if the options induced by the found partitions transfer to new tasks (Table \ref{table:1}). For puck stacking, a deep Q-network augmented with options from the previous tasks significantly outperforms both of our baselines. "Baseline" is a vanilla deep Q-network that does not retain any information from the initial task, whereas "Baseline, share weights" remembers the learned weights from the first task. Options are superior to the weights sharing baseline because they can take the agent to any desirable state, not just the goal. For instance, the 2 and 2 puck stacking task benefits from the option "make a stack of two pucks"; hence, options enable faster learning than weight sharing. We would also like to highlight one failure mode of the weight sharing baseline: the agent can sometimes get stuck repeatedly reaching the goal of the initial task without going any further. This behavior is exemplified in the transfer experiment from 2 puck stacking to 3 puck stacking. Here, the weight sharing agent continually places two pucks on top of one another, then lifts the top puck and places it back, which leads to slower learning than in the no-transfer baseline. Options do not suffer from this problem.

As reported in Figure \ref{fig:fig_5}, the learned partitions for the stairs building and component arrangement tasks underperform compared to puck stacking. Regardless, we observed a speed-up compared to the no-transfer baseline in all experiments except for the transfer from stairs from 3 pucks to 3 puck component. Options also outperform weight sharing in 3 out of 5 experiments with the non-optimal partitions, albeit not significantly.

\subsection{Discrete blocks world}\label{discrete_blocks_world}

Finally, we compare our partitioning algorithm to the decision tree method from \cite{wolfe2006} in the blocks world environment. The environment consists of three blocks that can be placed in four positions. The blocks can be stacked on top of one another, and the goal is to place a particular block, called the \textit{focus block}, in a goal position and height. With four positions and three blocks, 12 tasks of increasing difficulty can be generated. The agent is penalized with -1 reward for each action that does not lead to the goal; reaching the goal state results in 100 reward.

Although a neural network can learn to solve this task, a decision tree trains two orders of magnitude faster and often reaches better performance. We used a decision tree from the scikit-learn package \cite{scikit_learn} with the default settings as our $g$ classifier. All modifications from Subsection \ref{sub_speeding_up} specific to a neural network were omitted: class balancing and confidence thresholding. We also disabled the state-action block size threshold because the number of unique transitions generated by this environment is low and the decision tree does not make many mistakes. Despite the decision tree reaching high accuracy, we set a limit of 100 state-action blocks to avoid creating thousands of state-action pairs if the algorithm fails. The abstract MDP was recreated every 3000 time steps and the task terminated after 15000 time steps.

Figure \ref{fig:fig_3} compares the decision tree version of our algorithm with the results reported in \cite{wolfe2006}. There are several differences between our experiments and the algorithm in \cite{wolfe2006}: Wolfe's algorithm works with a Controlled Markov Process (CMP), an MDP augmented with an output function that provides richer supervision than the reward function. Therefore, their algorithm can start segmenting state-action blocks before it even observes the goal state. CMPs also allow an easy transfer of the learned partitions from one task to another; we solve each task separately. On the other hand, each action in Wolfe's version of the task has a 0.2 chance of failure, but we omit this detail to satisfy the assumptions of our algorithm. Even though each version of the task is easier in some ways are harder in others, we believe the comparison with the only previous algorithm that solves the same problem is valuable.

\begin{figure}[t]
    \includegraphics[width=0.44\textwidth]{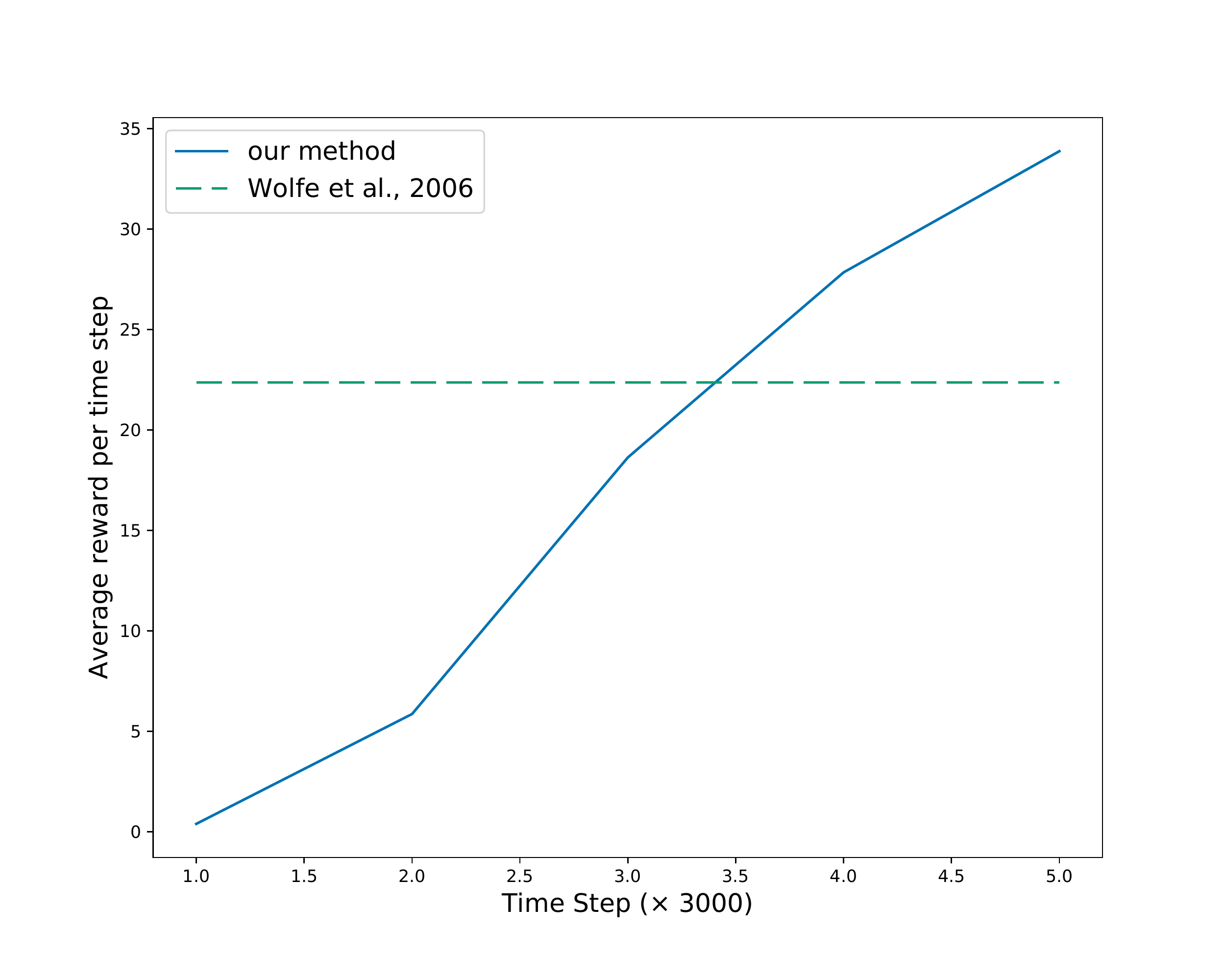}
    \caption{Comparison with Wolfe et al. \cite{wolfe2006} in the Blocks World environment. The horizontal line marks the highest mean reward per time step reached by Wolfe et al. We averaged our results over 100 runs with different goals.}
    \label{fig:fig_3}
\end{figure}

\subsection{Discussion}\label{exp_discussion}

We return to the questions posited at the beginning of this section. Online Partition Iteration can find the right partition of the state-action space as long as the individual state-action blocks are clearly identifiable. For tasks with more complex classification boundaries, the partitions found are suboptimal, but still useful. We showed that options speed-up learning and outperform the baselines in the majority of the transfer experiments. Our algorithm also outperformed the only previous method of the same kind \cite{wolfe2006} in terms of finding a consistent partition.

The main drawback of Online Partition Iteration is that is it highly influenced by the accuracy of the classifier $g$. During state projection, it takes only one incorrectly classified action (out of 12544 actions used in our pucks world experiments) for the state block classification to be erroneous. Confidence thresholding helps in the task of stacking pucks (Figure \ref{fig:fig_5}b), as it can filter out most of the errors. However, trained classifiers for the other two tasks, arranging components and building stairs, often produce incorrect predictions with a high confidence.

Moreover, the errors during state projection get amplified as the partitioning progresses. Note that the dataset of state-action pairs (inputs) and state-action blocks (classes) is created based on the previous state partition, which is predicted by the classifier. In other words, the version of the classifier $g$ at step $t$ generates the classes that will be used for its training at step $t + 1$. A classifier trained on noisy labels is bound to make even more errors at the next iteration. In particular, we observed that the error rate grows exponentially in the number of steps required to partition the state-action space.

In these cases, the partitioning algorithm often stops because of the limit on the number of state-action blocks (10 for the pucks domain). That is why the performance for the component arrangement and stairs building tasks is not sensitive to the state-action block size threshold (Figure \ref{fig:fig_5}a). Nevertheless, these noisy partitions also help with transfer learning, as shown in Table \ref{table:1}.

\section{Conclusion}

We developed Online Partition Iteration, an algorithm for finding abstract MDPs in discrete and continuous state spaces from experience, building on the existing work of Givan et al., Ravindran and Wolfe et al. \cite{givan2003,ravindran2004,wolfe2006}. We proved the correctness of our algorithm under certain assumptions and demonstrated that it could successfully abstract MDPs with high-dimensional continuous state spaces and high-dimensional discrete action spaces. In addition to being interpretable, the abstract MDPs can guide exploration when learning new tasks. We created a transfer learning method in the framework of options \cite{sutton1999}, and demonstrated that it outperforms the baselines in the majority of experiments.

\clearpage

\bibliographystyle{ACM-Reference-Format}  
\balance  
\bibliography{main}  

\end{document}